\documentclass[11pt]{article}

\usepackage{amsmath}
\usepackage{amssymb}
\usepackage{latexsym}
\usepackage{epsfig}
\usepackage{color}
\usepackage{enumerate}
\usepackage[margin=1in]{geometry}

\newtheorem{theorem}{Theorem}
\newtheorem{lemma}{Lemma}

\newlength{\boxwidth}
\makeatletter
\DeclareRobustCommand{\qed}{%
  \ifmmode 
  \else \leavevmode\unskip\penalty9999 \hbox{}\nobreak\hfill
  \fi
  \quad\hbox{\qedsymbol}}
\newcommand{\openbox}{\leavevmode
  \hbox to.77778em{%
  \hfil\vrule \vbox to.675em{\hrule width.6em\vfil\hrule}%
  \vrule\hfil}}
\newcommand{\qedsymbol}{\openbox}

\newenvironment{proof}[1][\proofname]{\par \normalfont
  \topsep6\p@\@plus6\p@ \trivlist
  \item[\hskip\labelsep\bfseries\itshape #1.]\ignorespaces }{%
  \qed\endtrivlist }

\newcommand{\proofname}{Proof}

\newcommand{\hide}[1]{}

\def\J{\hbox{\rm JOIN}}
\def\L{\hbox{\rm LINK}}
\def\PJ{\hbox{\rm PJOIN}}
\def\E{{\sf E}}

\title{Unsupervised Learning Through Prediction in a Model of Cortex}

\author{
Christos H. Papadimitriou\thanks{UC Berkeley. {\tt christos@cs.berkeley.edu}} 
\and 
Santosh S. Vempala\thanks{Georgia Tech. {\tt vempala@gatech.edu}}\\
}

\begin{document}

\maketitle

\begin{abstract}
We propose a primitive called PJOIN, for ``predictive join,'' which combines and extends the operations JOIN and LINK, which Valiant proposed  as the basis of a computational theory of cortex.  We show that PJOIN can be implemented in Valiant's model. We also show that, using PJOIN, certain reasonably complex learning and pattern matching tasks can be performed, in a way that involves phenomena which have been observed in cognition and the brain, namely memory-based prediction and downward traffic in the cortical hierarchy.
\end{abstract}

\section{Introduction}  
Human infants can do some amazing things, and so can computers, but there seems to be almost no intersection or direct connection between these two spheres of accomplishment.   In Computer Science we model computation through algorithms and running times, but such modeling quickly leads to intractability, even when applied to tasks that are very easy for humans.  The algorithms we invent are clever, complex and sophisticated, and yet they work in fashions that seem completely incompatible with our understanding of the ways in which the brain must actually work --- and this includes learning algorithms.  Accelerating advances in neuroscience have expanded tremendously our understanding of the brain, its neurons and their synapses, mechanisms, and connections, and yet no overarching theory appears to be emerging of brain function and the genesis of the mind.   As far as we know, and the spectacular successes of neural networks \cite{Hinton2006,Lecun1998} notwithstanding, no algorithm has been proposed which solves some nontrivial computational problem in a computational fashion and style that can be credibly claimed to reflect what is happening in the brain when the same problem is solved. This is the scientific context that is motivating the present work.

We believe that formal computational models and algorithms can play a valuable role in bridging these gaps.  We have been inspired in this regard by the pioneering work of Les Valiant on {\em the neuroidal model}, a random directed graph whose nodes and edges are capable of some very local computation \cite{Valiant95}.  We have also been prompted and influenced by certain interesting facts and conjectures about the brain's operation, which have been emerging in recent years.  First, even though the visual cortex, for example, processes sensory input in a hierarchical way, proceeding from areas lower in the hierarchy to higher areas (from retina to V1, from there to V2, to V4, to MT, etc.), there seem to be many neuronal connections, and presumably significant firing traffic, directed {\em downwards} in this hierarchy (see e.g., \cite{Lamme00} and subsequent work).  Second, there seems to be an emerging consensus among students of the brain that {\em prediction} may be an important mode of brain function \cite{Churchland94, Hawkins04, Rao99}:  For example, vision involves not just of passive hierarchical processing of visual input, but also active acquisition of new input via saccades (rapid forays to new parts of the visual field), whose purpose seems to be to verify predictions made by higher layers of the visual cortex; significantly, many eye muscles are controlled by low areas of the visual hierarchy. Third, the connections in the cortex seem to deviate from the standard $G_{np}$ model of random graphs, in that reciprocal connections (closing cycles of length two) and transitive connections (connecting nodes between which there is a path of length two {\em ignoring directions\/}) are much more likely than $G_{np}$ would predict \cite{Song05}.  We briefly discuss a new model of brain connectivity in Section \ref{sec:reciprocity}, as well as ways in which such an assumption complements our argument.

Valiant models the cortex as a random directed graph of {\em neuroids} (abstract neuron-like elements) connected via directed edges called {\em synapses}.  The algorithms running on this platform are {\em vicinal}, by which it is meant that they are local in a very strict sense: the only communication from other neuroids that a neuroid is allowed to receive is the sum of the action potentials of all currently firing neuroids with synapses to it.  This is the only outside information a neuroid can use in deciding what to do next.  Valiant goes on to posit that real-world objects or concepts can be represented by {\em sets of neuroids}.  There is a tradition of using such sets of neurons, or equivalently sparse binary vectors, as the basic formal unit of brain function \cite{Kanerva1988}, but Valiant's is, in our reading, the most explicitly algorithmic such framework.  If one assumes that the underlying directed graph is random in the sense of the $G_{np}$ model of random graphs \cite{Erdos60}, and taking into account what we know about the number of neurons in the cortex, the number of synapses out of (or into) a neuron, and the number of firing neurons that can cause another to fire, Valiant argues carefully \cite{Valiant95} that it is possible to choose some number $r$, perhaps  between $10$ and $100$, such that sets of $r$ neurons can usefully represent a concept.  Such sets are called {\em items}, and they constitute the basic element of Valiant's theory: all $r$ neuroids of an item firing (or perhaps some overwhelming majority thereof) is coterminal with the corresponding concept being ``thought about.''   Valiant shows that such items can be combined through the basic operations of JOIN and LINK, creating compound new items from, respectively simply connecting, existing items.  For example, if $A$ and $B$ are items (established sets of about $r$ neuroids) then Valiant shows how to create, via a vicinal algorithm, a new item $\J(A,B)$ which stands for the combination, or conjunction, of the two constituent items.  

Our main contribution is the formulation and implementation of a new operation on items that we call {\em predictive join} or {\em PJOIN}.  This operation PJOIN extends JOIN in that, if only one of the constituent elements of $\PJ(A,B)$ fires, say $A$ fires, then the structure will {\em predict $B$,} and initiate a sequence of downstream events whose purpose is to check this prediction.  We show (Theorem 1) that the PJOIN of two items --- which may be, and typically will be, themselves PJOINs --- can be created by a vicinal algorithm so that it functions according to certain precise specifications, capturing the discussion above.  

We suspect that the power of cortical computation lies not so much in the variety and sophistication of the algorithms that are actually implemented in cortex --- which might be very limited ---  but in the ways in which the cortex interacts with an appropriately interesting environment through sensors and actuators.  In other words, cortical architecture may happen to be very well adapted to the kinds of environments mammals live in.  Here we use our PJOIN operation --- {\em and almost nothing else} --- to design a vicinal algorithm for performing a basic task that would be required in any such interaction with an environment:  unsupervised learning of a set of patterns.  By {\em learning} we mean that, if a pattern is presented to a sensory device for long enough, then our vicinal algorithm will ``memorize it,'' in that it will create a structure of items connected by PJOINs which represent this pattern.  If many patterns are presented, one after the other and with repetitions, all these patterns will be memorized in a shared structure.  If sometime in the future one of the already presented patterns is presented again, then with high probability the structure will ``recognize it,''  by which we mean that there will be one special item (set of neuroids)  whose firing will signal the act of recognition of that particular item. 
Moreover, a pattern that has already been presented once needs to be presented for a considerably shorter time on subsequent presentations.
While $\J$ can be viewed as the {\em unsupervised memorization} of two items in the neuronal model, $\PJ$ allows unsupervised memorization of patterns of arbitrary length. As we note in Section \ref{sec:learn}, using $\J$ directly for more than two items does not yield this desired consistency in behavior. 

We simulated these algorithms and measured their performance on binary patterns, including the sizes of representations created and updates to them, the total traffic generated in each presentation, and the fraction of traffic that is downward.  The results, for varying ranges of parameters, are presented in Section \ref{sec:simulation}.  They are consistent with the theoretical predictions.  In particular, both memorization and recognition are accomplished in a rather small number of ``steps'' --- firings of individual neuroids --- typically in the low tens; this is an important desideratum, because we know that complex cognitive operations can be performed in quite a bit less than second, and thus in fewer than a hundred such steps.

Finally, in Section \ref{sec:control} we discuss several important issues surrounding our results, such as the nature of the rudimentary {\em algorithmic control} necessary for the running of these and other vicinal algorithms, a new enhanced model of random graphs that appears to agree better with recent insights about cortical connectivity, as well as important further steps for this research program.


\section{Valiant's Computational Model of Cortex}
The basic structure in Valiant's computational theory of cortex is an {\em item,} which represents in the brain a real-world object or concept.  Each item is implemented as a set of neurons.  Valiant considers two alternative representations, one in which these sets are disjoint and one in which they may overlap; here we shall work in the disjoint framework, but we see no reason why our work could not be extended to overlapping items.  Once such a set of neurons has been defined, the simultaneous firing of these neurons (or perhaps of some ``overwhelming majority'' of them) represents a brain event involving this item.

Valiant defined two key operations involving items.  The $\J$ operation creates, given two items $A$ and $B$, a new item $C=\J(A,B)$.  Once the new item has been created, every time both $A$ and $B$ fire, $C$ will fire in the next instant.  A second operation is $\L$.  Given two established items $A$ and $B$, the execution of $\L(A,B)$ has the effect that, henceforth, whenever item $A$ fires, $B$ will fire next.  Valiant showed that these two operations --- both the creation of these capabilities, and their application --- can be implemented by simple {\em vicinal algorithms}.  These algorithms are implemented in a particular formal {\em neuroidal model} of computation, involving only local computation, and communication strictly through action potentials.  It is argued in \cite{Valiant00,Valiant06} that vicinal algorithms realistically capture the capabilities of neurons.  Furthermore, it is argued in \cite{Valiant05, Valiant12} that even cascading usage of these two operations  (taking the $\J$ of two $\J$s, and then again, and so on) is up to a point a reasonably realistic hypothesis.  

The number $r$ of neurons representing an item depends on the numerical parameters of cortex, as well as its random nature.  Let $n$ be the total number of neurons, and let $d$ be the typical {\em out-degree} of a neuron, that is, the number of other neurons with which the neuron synapses. In the human brain, $n$ is believed to be roughly $10^{11}$ and $d$ roughly $10^4 - 10^5$.  Let $k$ be a number such that, if $k$ neighboring neurons of a particular neuron are firing, then these $k$ action potentials are typically enough to cause the neuron to fire. 
Valiant argues convincingly in \cite{Valiant05} that, given our current knowledge of the parameters $n$, $d$, and $k$ for various organisms, and assuming that the connections between neurons are random in the sense of $G_{np}$, a range exists for the new parameter $r$, such that $r$ neurons can usefully represent an item (see also \cite{FeldmanV09}).  The preceding discussion of the parameters $k$ and $r$ is an argument about the numerical plausibility of the neuroidal model, and does not mean that all items are represented by the same exact number of neurons, or that synapses are required to have the same strength and all neurons the same number of incoming synapses at all times --- as we shall see, a neuroid can even set its own threshold in this model, and a synapse its own strength.  However, for such a particular vicinal  algorithm to be plausible, the numerical range of these parameters should be reasonable; for example, the firing of a neuron caused by the firing of very few other neurons, or a single one (either because of a tiny threshold or a huge synaptic strength) would be unreasonable.  In the algorithms of this paper we do allow the parameters $r$ and $k$, and even the firing threshold, to vary by a small factor in certain exceptional cases.  Another desideratum on vicinal algorithms is that, even though the model assumes discrete time steps of a global clock, one should not rely too much on this.  In other words, such algorithms should not have very long chains of consecutive steps requiring strict synchronization between many neuroids.  Only sequences of up to three or four synchronized steps are used here (in \cite{Valiant95,Valiant00,Valiant06} only two synchronized steps were needed).   Naturally, it is necessary to assume that all neuroids comprising an item can fire simultaneously, in response to other items firing.

\subsection*{Vicinal Algorithms}
The substrate of vicinal algorithms is a directed graph whose vertices are {\em neuroids}, (idealized neurons) connected through directed edges called {\em synapses}.  At every discrete time $t \geq 0$, a neuroid $v_i$ is at a {\em state} $(T^t_i, f^t_i, q^t_i)$ consisting of three components: a positive real number $T^t_i$ called its {\em threshold}; a Boolean $f^t_i$ denoting whether $i$ is {\em firing} presently; and another memory $q^t_i$ which can take finitely many values.  In our algorithms, $T$ will for the most part be kept constant.  Similarly, every synapse $e_{ji}$ from neuroid $v_j$ to neuroid $v_i$ has a state $(w^t_{ji},qq^t_{ji})$ consisting of a real number $w^t_{ji}$ called its {\em strength;} plus another finitary memory $qq^t_{ji}$.  When the superscript $t$ is understood from context, it will be omitted.  

The operation of vicinal algorithms is strictly local: neuroids and synapses update their own state in discrete time steps.  The only way that the operation of neuroid $v_i$ can be influenced by other neuroids is through the quantity $W_i = \sum_{j: f_j =1} w_{ji}$, the sum total of incoming action potentials.  At each time instant, the firing status of all neuroids is updated:
$$W_i^t \leftarrow \sum_{j: f_j^t = 1} w^t_{ji};\ \ {\rm if\ } W^t_i \geq T^t_i \ {\rm then\ } f^{t+1}_i \leftarrow 1,$$
and after that the remaining of the state of the neuroids, and the state of the synapses, are updated:
$$(T_i^{t+1}, q_i^{t+1}) \leftarrow \delta(T_i^t,q_i^t,f^{t+1}_i,W^t_i),$$
$$(w^{t+1}_{ji}, qq^{t+1}_{ji}) \leftarrow \lambda(q_i^t,qq^t_{ji},f^t_j,f^{t+1}_i,W^t_i).$$
That is, at each step $f_i$ is set first --- this is tantamount to assuming that the action potentials of other neuroids are transmitted instantaneously --- and then the remaining state components of the neuron's state, and also of its incoming synapses, are updated in an arbitrary way.  Notice that the update functions $\delta$ and $\lambda$ are {\em the same for all neuroids $v_i$ and synapses $e_{ji}$}.  

This is the repertoire of actions of vicinal algorithms.  For the algorithms to happen, some rudimentary control is necessary, causing the synchronous firing of certain neuroids, for example in the initiation of the creation of $\J(A,B)$ explained next.  We discuss the plausible vicinal implementation of such control in Section \ref{sec:reciprocity}. 

In our descriptions of our vicinal algorithms, instead of presenting explicit functions $\delta$ and $\lambda$, we shall use steps manipulating the {\em total state} of each neuroid, by which we shall mean all local information, namely the memory and threshold of the neuroid, plus the memory and strengths of all of its incoming synapses.

As an example of a vicinal algorithm, we next explain how the JOIN and LINK operations are performed; our exposition follows the one in Valiant \cite{Valiant12}.

\subsection*{The JOIN and LINK Operations}  Given two items $A$ and $B$, each represented by $r$ neuroids, the operation $\J(A,B)$ can be implemented in two steps, as follows:

\begin{itemize}
\item The neuroids that will represent $C=\J(A,B)$ will be recruited from a population of neuroids such that there are expected (on the basis of the random graph properties of the neuroidal network)  to exist at least $k$ synapses going from each of $A,B$ to at least $r$ of these neuroids, for some desired parameters $r,k$.  These neuroids are initialtly at a  total state that we call {\sc Candidate}: not firing, all memories at some null state, the standard threshold $T$, and all strengths of incoming synapses equal to $T\over k$.  

\item At the first step, all neuroids in $A$ fire.  If, at the end of this step, a candidate neuroid fires (and therefore it had at least $k$ synapses coming from $A$), then its total state becomes one that we call {\sc poised}: all synapses that come from firing $A$ neuroids have now strength $T^2\over 2kW_i$ (so that if all neuroids of $A$ fire again, they will together achieve a sum of $T\over 2$; recall that $W_i$ is the total strength of all synapses from $A$, and therefore $kW_i\over T$ denotes the number of incoming synapses from $A$).  All candidates that did not fire enter a {\sc Dismissed} total state with all incoming strengths zero: they will not be needed to form $C$.

\item Then in the next step all neuroids of $B$ fire.  If a poised neuroid fires in response, then it enters an {\sc Operational} state, in which it is ready to partake in the $\J(A,B)$ operation, should both $A$ and $B$ fire simultaneously in the future.  These are the neurons which will henceforth comprise item $C=\J(A,B)$.  All synapses from $B$ have strength $T^2\over 2kW_i$ where $W_i$ denotes the new sum of strengths, while the neuroids from $A$ retain their strength (and thus, if now all neuroids in $A$ and $B$ fire now, all neuroids in $C$ will fire).   All poised neuroids that did not fire in this second step are now {\sc Dismissed}.
\end{itemize}

We can summarize the vicinal algorithm for the creation of  $\J(A,B)$ as follows:
{
\begin{itemize}
\leftskip 4cm \item[Step 1]    $A\Rightarrow$ fire;\ \ \  not $A$ and not $B\Rightarrow$ {\sc Candidate}
\item[Step 2]    {\sc Candidate} and {\sc Fired} $\Rightarrow$ {\sc Poised};\ \ \ $B\Rightarrow$ fire
\item[After Step 2] {\sc Poised} and {\sc Fired} $\Rightarrow$ {\sc Operational}; {\sc Poised} and not {\sc Fired} $\Rightarrow$ {\sc Dismissed}
\end{itemize}
}

The operation $\L(A,B)$ which, once executed, causes all neuroids of $B$ to fire every time $A$ fires, is implemented in a similar manner, by employing a large population of intermediate {\em relay} neuroids \cite{Valiant12} that are shared by all linking pairs.

\begin{itemize}
\leftskip 4cm \item[Step 1]    $A\Rightarrow$ fire;\ \ \  $B\Rightarrow$ {\sc Prepared};
\item[Step 2]    Some relay neurons fire as a result of $A$'s firing at Step 1; the neurons of $B$ will likely all fire as a result of this.
\item[After Step 2] {\sc Prepared} and {\sc Fired} $\Rightarrow$ {\sc L-Operational}
\end{itemize}
Here by {\sc L-Operational} it is meant that all incoming synapses from relay neuroids which have fired have memory state {\sc L-operational} and strength equal to $T/k$, where $T$ is the threshold and $k$ is a small enough integer so that we expect there to be $k$ relay nodes which both fire upon $A$ firing {\em and} synapse to all neurons of $B$.

\subsection*{Refraction}
Neurons in the brain are known to typically enter, after firing, a {\em refractory period} during which new firing is inhibited.   In our model, the firing of a neuroid can be followed by a refractory period of $R$ steps.  This can be achieved through $R$ memory states, say by increasing the neuroid's threshold $T$ by a factor of $2^{R-1}$, and subsequently  halving it at every step.  {We shall routinely assume that, unless otherwise stated, the firing of an item will be followed by one refractory step.}


\section{Predictive JOIN}
We propose now an enhanced variation of JOIN, which we call {\em predictive JOIN}, or $\PJ(A,B)$.  As before, every time $A$ and $B$ fire simultaneously, $C$ will also fire.  In addition, if only $A$ fires, then $C$ will enter a state in which it will ``predict $B$,''  in which $C$ will be ready to fire if just $B$ fires.  Similarly, a firing of $B$ alone results in the prediction of $A$.  The effect is that the firing of $A$ and $B$ now need not be simultaneous for $C$ to fire.  

But there is more.  Having predicted $B$, $C$ will initiate a process whereby its prediction is checked:  Informally, $C$ will ``mobilize'' a part of $B$.  If $B$ also happens to be the predictive JOIN of some items $D$ and $E$, say (as it will be in the intended use of PJOIN explained in the next section), then parts of those items will be mobilized as well, and so on.  A whole set of items reachable from $B$ in this way (by what we call {\em downward firings\/)} will thus be searched to test the prediction.  As we shall see in the next section, this variant of JOIN can be used to perform some elementary pattern learning tasks.

\subsection*{Implementation of PJOIN}
The vicinal algorithm for creating $\PJ(A,B)$ is the following:

\paragraph{Create $\PJ(A,B)$}
\begin{enumerate}
\item First, create $C=\J(A,B)$, by executing the vicinal algorithm of the previous section.  
\item Each neuroid in $C$ participates in the next steps with probability half.  Valiant \cite{Valiant00}  allows vicinal algorithms to execute randomized steps, but here something mush simpler is required, as the random bit needed by a neuroid could be a part of that neuroid, perhaps depending on some component of its identity, such as its layer in cortex.  We shall think of the chosen half (in expectation)  of the neuroids of item $C$ as comprising a new item which we call $C_P$, for ``predictive $C$.''  That is, the neurons of item $C_P$ are a subset of those of item $C$.  (Strictly speaking this is departure from the disjoint model of implementing items, but it is a very inessential one.) The neuroids of $C$ that are not in $C_P$ now enter a total state called {\sc Operational}, ready to receive firings from $A$ and/or $B$ as any neuroid of $\J(A,B)$ would.

\item Suppose now that $A$ and $B$ are predictive joins as well, and therefore they have parts $A_P$ and $B_P$ (this is the context in which $\PJ$ will be used in the next section).  In the next two steps we perform $\L(C_P,A_P)$ and $\L(C_P,B_P)$, linking the predictive part of $C$ to the predictive parts of its constituent items; there is no reason why the two LINK operations cannot be carried out in parallel.  In the discussion section we explain that this step can be simplified considerably, omitting the LINK operations, by taking into account certain known facts about the brain's connectivity, namely the {\em reciprocity} of synapses.  After these LINK creations, the working synapses from the relay neuroids to the neuroids of $A_P$ and $B_P$ will be in a memory state {\sc Parent} identifying them as coming ``from above.'' Such synapses will have a strength larger than the minimum required one (say, double), to help the neurons identify firings from a parent.

\item After this is done, $C_P$ enters the total state {\sc P-Operational}, where all strengths of synapses from $A$ and $B$ to $C_P$ (but {\em not} to the rest of $C$) are doubled.  The effect is that the neuroids of $C_P$ will fire if {\em either} $A$ fires, {\em or} $B$ fires, {\em or both}.  This concludes the creation of $\PJ(A,B)$. Notice that it takes four steps (we discuss in the last section one possible way to simplify it to three steps).  After the creation of the predictive join is complete, indeed, if only one of $A$ or $B$ fires, then $C_P$ will fire, initiating a breadth-first search for the missing item.  
\end{enumerate}

To summarize, we show below the vicinal algorithm, which we call {\bf Create $\PJ(A,B)$}:
\begin{itemize}
\leftskip 4cm \item[Steps 1 and 2]  $C=\J(A,B)$;
\item[After Step 2]  $C_P$ consists of about half of the neurons in $C$. $C$ and not $C_P\Rightarrow$ {\sc Operational} 
\item[Steps 3 and 4]  $\L(C_P,A_P); \L(C_P,B_P)$
\item[After Step 4]  $A_P,B_P\Rightarrow$ {\sc L-operational} synapses in {\sc  Parent} memory state; \\ $C_P\Rightarrow$ {\sc P-Operational} (synapses from $A$, $B$ double their strength).
\end{itemize}

The {\em operation} of $\PJ(A,B)$ is a more elaborate version of that of JOIN:
\begin{enumerate}
\item If both inputs $A$ and $B$ fire simultaneously, then $\PJ(A,B)$ will operate as an ordinary JOIN.  
\item  As soon as one of $A$ and $B$ fires --- suppose that $A$ fires --- then $C_P$ fires (because of the doubled synaptic strengths), an event that will cause $B_P$ to fire downwards in the next step.  Notice that $A_P$ will not fire as a result of $C_P$ firing, because it is in a refractory state.  Also, after the firing of $A$, the neuroids in $C$ double the strength of the synapses coming from $B$ (that is, from neuroids that do not come from a parent and are refractory), and all of $C$ enters a total state called {\sc Predicting $B$}.  Thus, if henceforth the predicted item $B$ fires, all of $C$ will fire.  

\item $C_P$'s firing will cause $B_P$ to fire downwards in the next step; $A_P$ does not fire as a result of $C_P$ firing, because it is in a refractory state.  After this, $C_P$ enters a {\sc Passive} total state, in which it will ignore firings of the $E_P$ part of any item $E=\PJ(C,D)$ further above.  Again, this is accomplished by setting the strength of these synapses (which are the ones in the {\sc Parent} memory state) to zero.

\item If henceforth the predicted item $B$ fires, then $C$ fires and all of $A,B,C$ go back to the {\sc Operational} ({\sc P-operational} for $C_P$) total state.  

\item Finally, if one of $C$'s parents (by which we mean any item of the form $\PJ(C,D)$) fires ``downwards,'' then 
the neurons in $C_P$ will fire (we think of them as firing ``downwards''), thus propagating the search for the predicted items. 
\end{enumerate}

The state diagram of the desired operation of $\PJ(A,B)$ is shown in Fig. \ref{fig:pjoin-states}.

\begin{figure}[h]
\centering 
\includegraphics[width=4in]{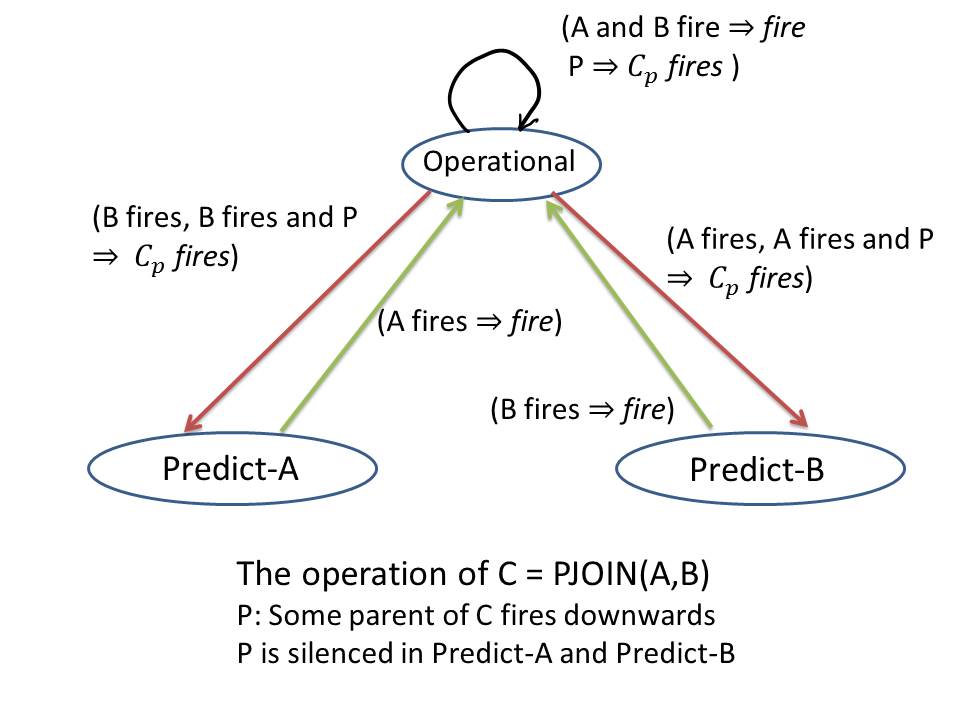}
\caption{The {\em operational} state transition diagram for an item created by PJOIN}
\label{fig:pjoin-states}
\end{figure}

\begin{theorem}
There is a vicinal algorithm that creates in four steps an item $\PJ(A,B)$ operating as specified by the diagram in Figure \ref{fig:pjoin-states}.
\end{theorem}
\begin{proof}
We must argue two things:  First, that the algorithm {\bf Create $\PJ(A,B)$} above is vicinal, that is, it can be implemented within the model described in Section 2 through appropriate state transformation functions $\delta$ and $\lambda$.  And second, that the resulting item responds to firings of other items as specified by the diagram.

For the creation part, and since we assume that JOIN and LINK are implemented correctly, all we need to show is that the conditions that must fold after the second and after the fourth step of the algorithm {\bf Create} $\PJ(A,B)$ can be implemented vicinally.  The former is the identification of a random subset $C_P$ of $C$, which is straightforward as discussed.   The latter consists of (a) the {\sc Parent} memory state of the synapses from $C_P$ to $A_P$ and $B_P$, and the doubling of the strengths of these same synapses.  Both (a) and (b) can be accomplished by the function $\lambda$ of the neuroids right after the creation of $\L(C_P.A_P)$ and $\L(C_P.B_P)$.

To argue that the resulting $\PJ(A,B)$ operates as specified in the diagram, it sufficed to check each transition It is easy to see that the transition graph is complete).  For example, for the transitions out of the {\sc Operational} state:  The ``$A$ and $B$ fire'' self-loop follows from the fact that, in this case, PJOIN operates exactly as JOIN.  The ``P'' self-loop is also immediate, since, for each parent $D$ of the item, the synapses coming from $D_P$ will cause the item to fire.  And so on for the other transitions. 
\end{proof}

\paragraph{Remark:  Control and Synchrony.}
Vicinal algorithms should run with minimum control and synchrony.  By {\em control} we mean firings initiated by outside events, other than the firing of neuroids explicitly mentioned in the algorithm.  The only control needed in our vicinal algorithms so far is the initiation of the firings of $A$ and $B$, in consecutive steps, during the creation of $\J(A,B)$.   In Section \ref{sec:reciprocity} we shall discuss briefly how all control necessary for vicinal algorithms can actually be implemented vicinally.

Similarly, expecting billions of neurons to operate in precise lockstep is not realistic, and so vicinal algorithms should not require much synchrony.  The only synchronization needed in our algorithms, besides the synchronous firing of all neurons of an item in response to the neurons of another item firing (a necessary assumption for Valiant's model, and whose plausibility  we discuss in Section \ref{sec:reciprocity}), is between the firing of two items in the two steps of the creation of $\J(A,B)$ and the two steps of $\L(A,B)$ (and we shall soon argue that the latter one may not be necessary).


\section{Pattern Memorization and Recognition}\label{sec:learn}
\def\zero{\hbox{\bf 0}}
\def\one{\hbox{\bf 1}}
Cortical activity acquires meaning only when it interacts with the world via appropriate sensors and actuators.
In this section we show that the primitive $\PJ$ enables a neurally plausible algorithm for memorizing and recognizing patterns.  

\paragraph{Sensing the environment.}
By {\em pattern} we mean a binary vector with $n$ coordinates, that is, a member of $\{0,1\}^n$ for some parameter $n$.  Visual patterns are only one of the many possibilities.
In particular, we shall assume that there are $n$ special items $S_1,\ldots,S_n$ called {\em sensors}.  Each $S_i$ has two special memory states, {\sc Zero} and {\sc One}, so that the memory states of all $n$ sensors ``spell'' the pattern.

Associated with each sensor $S_i$ there are two other items which we denote by $\zero_i$ and $\one_i$, such that:
\begin{itemize}
\item The states of the $n$ sensors will remain unchanged during a period that we call {\em pattern presentation}.   During such presentation, each sensor fires at most one --- either spontaneously at a random time, or when predicted.


\item The only activity of the neuroids of sensor $S_i$ is the aforementioned firing, as well as the occasional change in memory state when a different pattern is presented. 
\end{itemize}

The memory states of the $n$ sensors typically remain unchanged from one time step to the next, displaying a pattern $x\in \{0,1\}^n$.  A maximum time period during which these memory states are unchanged is called a {\em presentation} of the corresponding pattern.  {We assume that, during such a presentation, the $n$ sensor items fire spontaneously at random, each at most once.}  After the presentation of one pattern concludes, the presentation of a different pattern begins, and so on.

\paragraph{Memorization and recognition.}

We say that pattern $x\in \{0,1\}^n$ has been {\em memorized} if, during its first presentation, a hierarchical structure of items is eventually formed with one top item $I(x)$ (that is, an item that does not participate in further PJOINs) which we think of  as ``representing'' the presented pattern, in that, if $x$ and $y$ are different patterns, $I(x)\neq I(y)$.  Notice that $x$'s first presentation may follow the presentation of many other patterns $y,z,w\ldots$, and so its memorization structure will be ``built'' on top of existing memorization structures of the other patterns.  Now if, in any subsequent presentation of the same pattern $x$, this special item $I(x)$ fires and is the highest-level item that fires, we  say that the pattern has been {\em recognized}.

\paragraph{Learning with PJOIN.}

Before describing the algorithm for learning, we make a few remarks on sensor items.
\begin{itemize} 
\item If sensor $S_i$ fires in state {\sc One}, item $\one_i$ fires next; similarly, if it fires in state {\sc Zero,}, $\zero_i$ fires.  
This can be achieved, for example, by LINKing $S_i$ to both $\zero_i$ and $\one_i$, and appropriately manipulating the strengths of the corresponding synapses according to the memory state ({\sc Zero} or {\sc One}) of $S_i$.  

\item The items $\zero_i$ and $\one_i$ participate in PJOINs with other items, and thus they form the``basis'' of the $\PJ$ structure which will be eventually formed in response to the current pattern presentation. By ``basis'' we mean that these are the only items in the structure that are not themselves $\PJ$s of other items (see Figure 2).  

\item There is one difference between these basis items and the other $\PJ$ items in the structure:   As we have seen, if $C=\PJ(A,B)$ then its predictive part, $C_P$, synapses to $A_P$ and $B_P$.  In contrast, if $C=\one_i$ or $\zero_i$ then $C_P$ synapses to the corresponding $S_i$ item.  Thus, every time  $\one_i$ or $\zero_i$ ``fires downwards'' $S_i$ fires next, completing the chain of predictions.  Importantly, it does so without refraction (because it may have to fire in the next step to verify the prediction).

\item If and when a sensor item is predicted by a parent $\PJ$, its sampling weight is increased, multiplied by a factor $M>1$.
\end{itemize}

Memorization and recognition of a presented pattern is achieved by the algorithm shown in Figure \ref{fig:learn}.  
We say that an item is {\em $\PJ$-eligible} if at least $D$ time steps have elapsed after it has fired, and none of its parents has fired for the past $D$ steps, where $D>0$ is a parameter.  Define the {\em level} of an item to be zero if out is an item of the form $\zero_i$ or $\one_i$; and otherwise, if it is $\PJ(A,B)$ to be one plus (recursively) the largest of the levels of $A$ and $B$.  In our algorithm, we allow the delay $D$ associated with an item to depend on the level of the item.  The selection of pairs to be PJOINed is as follows:  PJOIN-eligible items are sampled, each with probability $q$.  
Note that, even though the algorithm of Figure \ref{fig:learn} is depicted as running synchronously, the only synchrony that is necessary for its correctness is between the neuroids of the same item, and between pairs of items being joined.  Furthermore, the only control (firings of neuroids initiated by events outside the algorithm, that is, by events other than firings of other neuroids), is the spontaneous firing of sensor items and the selection of PJOIN-eligible pairs of items to be PJOINed, and the initiation of the creation of the PJOIN.  We shall discuss in the next section how such rudimentary control can be accomplished neuroidally.

\begin{figure}
\fbox{\parbox{\textwidth}{
LEARN(pattern $x$):\\[0.1in]
Initialization: set the memory state of item $S_i$ to {\sc Zero} if $x_i = 0$ and to {\sc One} if $x_i = 1$.\\[0.05in]
Repeat the following:
\begin{enumerate} 
\item  Sample sensor items which have not yet fired to fire next (with probability $p$).
\item  Sample  PJOIN-eligible pairs of items $(A,B)$ to initiate the creation of $C=\PJ(A,B)$ (with probability $q$).
\item  Carry out one step of all neuroids.
\end{enumerate}
until no neuroid fires or changes state.}}

\caption{Pattern memorization and recognition: sensor items and $\PJ$-eligible pairs are sampled}
\label{fig:learn}
\end{figure}

We now turn to establishing the correctness of the algorithm.  To state our main result of this section, we first review the pertinent parameters.  $p$ is the sampling probability of the sensors, and $q$ the probability of sampling a PJOIN-eligible item for PJOINing.  $D$ is the delay after firing for an item to become PJOIN-eligible; we assume that $D=2\ell+2$.  Patterns $x_1,\ldots,x_m \in \{0,1\}^n$ are presented with repetitions in some arbitrary order, and each presentation lasts at least $T$ steps, where $T\geq 4(\log n)/p$.  

\begin{theorem}\label{thm:learn}
Suppose that patterns $x_1,\ldots,x_m \in \{0,1\}^n$ are presented with repetitions in some arbitrary order, with each presentation lasting at least $T\geq 4\log n + (2 \log n)/p$ steps; during each presentation, sensors are sampled with probability $p$, while eligible pairs are sampled with probability $q$; and the eligibility delay $D$ satisfies $D\geq 2\ell+2$ for items at level $\ell$.  Then, with high probability, 
\begin{itemize}
\item At the first presentation of each pattern $x_i$, one particular highest-level item $I(x_i)$ is created, and during all subsequent presentations of $x_i$ this item fires.  
\item Items $I(x_1),\ldots,I(x_n)$ are all distinct.  
\item  Furthermore, after each pattern is presented $O(\log n + \log m)$ times, the $I(x_i)$ items do not participate in the creation of PJOINs.
\end{itemize}
\end{theorem}

\noindent
{\bf Remark 1.}
This consistent behavior, of essentially creating one $\PJ$ tree for each pattern (with trees having distinct roots but sharing subtrees),  is not achievable using only $\J$.  If mere $\J$s were used, then on each presentation of a pattern of length $n > 2$, a new $\J$ tree would be created, with no consolidation of prior knowledge.
Another benefit of $\PJ$ is that, with it, the needed length of subsequent presentations of an item is decreased, since predictions are likely to accelerate considerably the rate with which sensors fire. 

\noindent
{\bf Remark 2.} It is possible that, upon a presentation of an item which has already been memorized, new items may be created as PJOINs of old items. These new items should not necessarily be considered as undesirable ``bugs;'' they may be seen as alternative ways of parsing --- of understanding --- an existing pattern.

Turning to the proof of the theorem, we start with a bound on the maximum level of $\PJ$ items for a single pattern.
\begin{lemma}
Under the conditions of the theorem, with high probability, the maximum level of an item  is $O(\log n)$. 
\end{lemma}
\begin{proof}
Suppose $n$ sensory items are PJOIN-eligible, so that $n-1$ PJOINs happen. Order the PJOINs $1,2,\ldots, n-1$. Fix a sensory item $s$ and let $X_j$ be the $0$-$1$ random variable that indicates whether $s$ participates in the $j$'th PJOIN. Then the height of $s$ at the end of the process is 
$X = \sum_{j=1}^{n-1} X_j$. From this we have,
\[
\E(X) \le \sum_{j=1}^{n-1} \frac{1}{n-j} = H_{n-1} < 1+ \ln (n-1).
\]
The claim follows from the Chernoff bound for independent Bernoulli random variables (with possibly different expectations):
\[
\Pr(X - \E(X) > t \E(X)) \le \left(\frac{e^{t}}{(1+t)^{(1+t)}}\right)^{\E(X)} \le (e+1)^{-\E(X)}
\]
where we used $t=e$.
\end{proof}

\begin{proof}[Proof of Theorem \ref{thm:learn}]
First we note that any item at any level fires if and only if all its descendants fire in response to an input pattern. This is clear since an item 
fires if and only if both its children fire or it is an input level item activated by an input symbol. In the former case, this firing criterion recursively applied, leads to the set of all descendants of the original item, i.e., an item fires iff all its input level descendants fire.
Thus, if the descendants of an item $C$ are exactly the set of sensory items for a particular pattern $x$ then $C$ fires if and only if $x$ is presented. 

Next we claim that with high probability, when a new pattern $x$ is encountered, the algorithm creates a new highest level cell to represent this pattern, i.e., the first presentation of the pattern $x$ creates an item $I(x)$ whose descendants include all $n$ sensory items that fire for $x$. After all possible processing of items according to the operational PJOIN state transition rules (see Fig. \ref{fig:pjoin-states}), if there is currently no item whose descendants match the pattern $x$, there must be a set of items $S$ that are PJOIN-eligible. New items are then created by sampling a pair of items from $S$ performing a PJOIN and adding the resulting item to $S$. This process finishes with a single item after $|S|-1$ such PJOIN's. The descendants of this item correspond to the pattern $x$. 

Now consider the presentation of a pattern $x$ for which we already have a top-level item $I(x)$. All items whose descendants match the pattern will eventually fire. Thus, the original highest-level item for $x$ will eventually fire, as claimed. En route, it is possible that two items that have fired, but none of whose parent items have fired as yet, will try to form a new PJOIN.  First we bound the total number of rounds. Independent of the PJOIN creation, in each round, each sensor fires with probability $p$. Thus, with high probability, all $n$ sensor items fire within $(2\ln n)/p$ rounds. The additional number of rounds after all sensors fire is at most the depth of the tree, i.e., at most $4\ln n$ whp. So the total number of rounds, again with high probability, is at most $4\ln n+(2\ln n)/p$. Thus our bound on $T$ is adequate. 

Next,  on a presentation of $x$, consider any PJOIN item $A$ with a parent $B$ that is a descendant of $I(x)$. After item $A$ fires, it will become PJOIN-eligible after $D$ steps. But after $A$ fires, its parent $B$ predicts its sibling. The sibling is at level at most $\ell(A)$, and therefore it will fire within at most $2\ell(A)+2$ steps.
Therefore such an item $A$  will not be PJOIN-eligible after $D$ steps since its parent would have fired.


Now consider the point where all sensor items have been sampled and existing PJOIN items that will fire, including $I(x)$ have fired. There might be items that have fired but whose parents did not fire, items that are valid for $x$, but were created during the presentation of a different pattern $y$. These items can now form new PJOINs. However, after seeing each of the $m$ input patterns at least once, no new PJOIN items will be created at level 1. This is because each sensory item, for each pattern, has at least one PJOIN parent. After another round of seeing all $m$ input patterns, no PJOIN items are created at level 2.  If we  bound the maximum level of items, then this will bound the maximum number of rounds of seeing all $m$ patterns after which no new $\PJ$ items are created. 

In fact, the maximum level of an item will be $O(\log n + \log m)$. To see this, we we can either use Lemma \ref{lem:max-pjoins} or the following more direct argument (with a slightly weaker bound) to conclude that the maximum number of items at level $\ell$ is at most $mn/2^{\ell}$. Fix a pattern $i$ and suppose the first pattern and the $i$'th pattern share an $r$ fraction of the sensor items. Then at level $1$, some fraction of these $r n$ sensors on which they agree will form PJOINs among themselves. The sensors involved in such PJOINs will no longer form new PJOINs on presentation of patterns $1$ or $i$, since they already have PJOIN parents that will fire. The remaining common sensors could form new PJOINs that are valid for both $1$ and $i$ in later presentations of $1$ or $i$. But no matter how many such PJOINs (among the common items) are formed on each presentation, the total number that can be valid for both $1$ and $i$ at level $1$ is at most $r n/2$, since once a sensor participates in one such PJOIN (where the other sibling is also from the common set of r.n sensors), it will no longer do so. For all $m$ patterns, the number of PJOINs at level $1$ valid for pattern $1$ and any other pattern is at most $mn/2$. 

We continue this argument upward. Of the items at level $1$ (at most $rn/2$) that are valid for both $1$ and $i$, on a presentation of $i$, some pairs could form new PJOINs that are valid for both $1$ and $i$, if they do not already have such PJOIN parents with both children valid for both patterns, but no matter what order this happens, there can be at most $rn/4$ such PJOINs at level 2. This continues, with the number of $\PJ$ items halving at each level. Therefore, after $O(\log n + \log m)$ rounds of seeing all input patterns (in arbitrary order, and any positive number of times in each round), no further PJOINs are created and the top-level items are stable for each pattern.
\end{proof}

The next lemma gives a sharp bound on the number of PJOIN items valid for two different patterns. 
\begin{lemma}\label{lem:max-pjoins}
Let $r_i = 1 - {1 \over n} H(x_1,x_i)$ be the fraction of overlap between two patterns $x_1$ and $x_i$, where 
$H$ is the Hamming distance.
The expected number of PJOINs created at level $\ell$ that are valid (i.e.,  will fire) for patterns $x_1$ and $x_i$ is 
$(n/2)(r_i^4 + r_i^2(1-r_i^2)(1 + \frac{1}{(1+r)^2}))$ for $\ell =1$ and 
\[
\frac{n}{2^\ell} \cdot \left(r_i^{2^{\ell+1}} + r_i^{2^\ell}(1-r_i^{2^{\ell}}) \cdot \left(1 + \left(\frac{1+\left(\frac{1+ \left( \ldots \right)^2}{1+r_i^{2^{\ell-1}}} \right)^2}{1+r_i^{2^\ell}}\right)^2   \right)\right) < \frac{n}{2^{\ell}}.
\]
for higher $\ell$. Moreover, whp, the number of such PJOINs will be within a factor of 4 of this bound.
\end{lemma}

\begin{proof}
The proof of the expectation is by induction on $\ell$. Consider two patterns $x_1$ and $x_2$ and let $r$ be the fraction of sensors where they agree. Then on the first presentation of $x_1$, the number of PJOIN items created at level $\ell$ is $n/2^\ell$, and the expected fraction of them that are also valid for $x_2$ will be $r^{2^{\ell}}$ (since all constituent sensor items must be drawn from the common $r$ fraction). Now, on a presentation of $x_2$, there might be more PJOINs created that are valid for $x_1$. At level $1$, such PJOINs will come from sensors that are valid for both, but do not yet participate in PJOINs that are valid for both. The expected fraction of such sensors is $r-r^2 = r(1-r)$.  For each such sensor, to be part of a new PJOIN valid for $x_1$, the other sensor must also be picked from this set. The total fraction of sensors participating in the creation of PJOINS will be $1-r^2$, since an $r^2$ fraction of sensors already have PJOINs that will fire. Thus, the expected number of new PJOINs created at level $1$ will be 
\[
\frac{n}{2} \cdot r(1-r)\cdot \frac{r(1-r)}{1-r^2} = \frac{n}{2} \cdot r^2(1-r^2)\cdot \frac{1}{(1+r)^2}
\]
and thus the total number in expectation at level $1$ is as claimed. 

To compute the expected number of PJOINs created at the next level, we note that the first presentation of $x_1$ creates
$r^4 \cdot (n/4)$  PJOINs that are valid for both $x_1$ and $x_2$. When $x_2$ is presented, the PJOIN items at level one that become PJOIN-eligible are $r^2(1-r^2)(1 + (1/(r+1)^2)\cdot (n/2)$, with the first term coming from PJOINs that were created on the presentation of $x_1$ and the second from new PJOINs created during the first presentation of $x_2$. Thus the number of PJOINs created at level $2$ in expectation will be $n/4$ times 
\[
\frac{\left( r^2(1-r^2)\left(1 + \frac{1}{(1+r)^2}\right)\right)^2}{1-r^4} 
= r^4(1-r^4)\left(\frac{1+\frac{1}{(1+r)^2}}{1+r^2}\right)^2.
\]
Thus, the total number of PJOIN-eligible items at level $2$ in expectation is
\[
\frac{n}{4} \cdot \left(r^4 - r^8 + r^4(1-r^4) \left(\frac{1+\frac{1}{(1+r)^2}}{1+r^2}\right)^2\right) = \frac{n}{4}\left(r^4(1-r^4)f(2)\right).
\] 
Continuing, the expected number of PJOIN eligible items at level $3$ is $(n/8)$ times
\[
r^8(1-r^8) + \frac{\left(r^4(1-r^4)f(2)\right)^2}{1-r^8} = r^8(1-r^8)\left(1 + \left(\frac{f(2)}{1+r^4}\right)^2\right) = r^8(1-r^8)f(3).
\] 
At the $\ell$'th level, the expected number of PJOIN-eligible items is thus $(n/2^\ell)$ times
\[
r^{2^\ell}(1-r^{2^\ell})f(\ell) \quad 
\mbox{ where } f(\ell) = 1+\left(\frac{f(\ell -1)}{1+r^{2^{\ell-1}}}\right)^2.
\]
Adding the number of PJOINs at level $\ell$ created by the first presentation of $x_1$ that are valid for both $x_1$ and $x_2$ and participated in PJOINs valid for both at the next level on the first presentation of $x_1$, which is  
$(n/2^\ell)r^{2^{\ell+1}}$, gives the claimed bound. 

We proceed to show the inequality, namely that the above quantity, $r^{2^\ell}(1-r^{2^\ell})f(\ell)$, is bounded by $1-r^{2^{\ell+1}}$ for any $r \in [0,1]$, by induction on $\ell$. In fact, our induction hypothesis is the following:
\[
f(\ell) \le 1 + \frac{1}{r^{2^{\ell}}}
\]
It holds for $\ell=2$. Assuming the induction hypothesis for $\ell-1$,
\begin{align*}
f(\ell) &= 1+\left(\frac{f(\ell -1)}{1+r^{2^{\ell-1}}}\right)^2\\
&\le 1 + \left(\frac{1 + \frac{1}{r^{2^{\ell-1}}}}{1+r^{2^{\ell-1}}}\right)^2\\
&= 1 + \frac{1}{r^{2^\ell}}.
\end{align*}
Thus, $r^{2^{\ell}}(1-r^{2^{\ell}})f(\ell) \le (1-r^{2^\ell})(1+r^{2^\ell}) \le 1-r^{2^{\ell+1}}$. 

The high probability bound follows from the independence of the PJOINs at each level.
\end{proof} 


\section{Experiments}\label{sec:simulation}

We implemented PJOIN and tested it on binary input patterns. Besides checking for correctness, we measured the total number of items created (size of memory), as well as the total traffic and the number of downward predictions, for different values of pattern length $n$, number of patterns $m$, the number of inputs presentations $T$, the delay $D$, sensor item sampling probability $p$, and PJOIN-eligible item sampling probability $q$. 

While we let each presentation run for as many steps as it took for every sensory item to be sampled, for $100$-size patterns, no presentation lasted more than 50 steps and after a few presentations, this dropped to $15$ or less steps. 
{\em Downward} traffic is indeed a substantial fraction of all traffic and increases with the number of distinct patterns presented. Figure \ref{fig:downward}(a) shows the the percent of downward traffic, as we increase the number of patterns $m$, for different values of pattern size $n$, with $p=q=0.1$ and $T=100$ steps per presentation. Fig \ref{fig:downward}(b) is the downward traffic with higher sensor item sampling probability $p$, keeping $q=0.1$ and $n=40$. These results were obtained without any restriction on the level difference between two items being PJOINed.

\begin{figure}
\includegraphics[width=8.5cm]{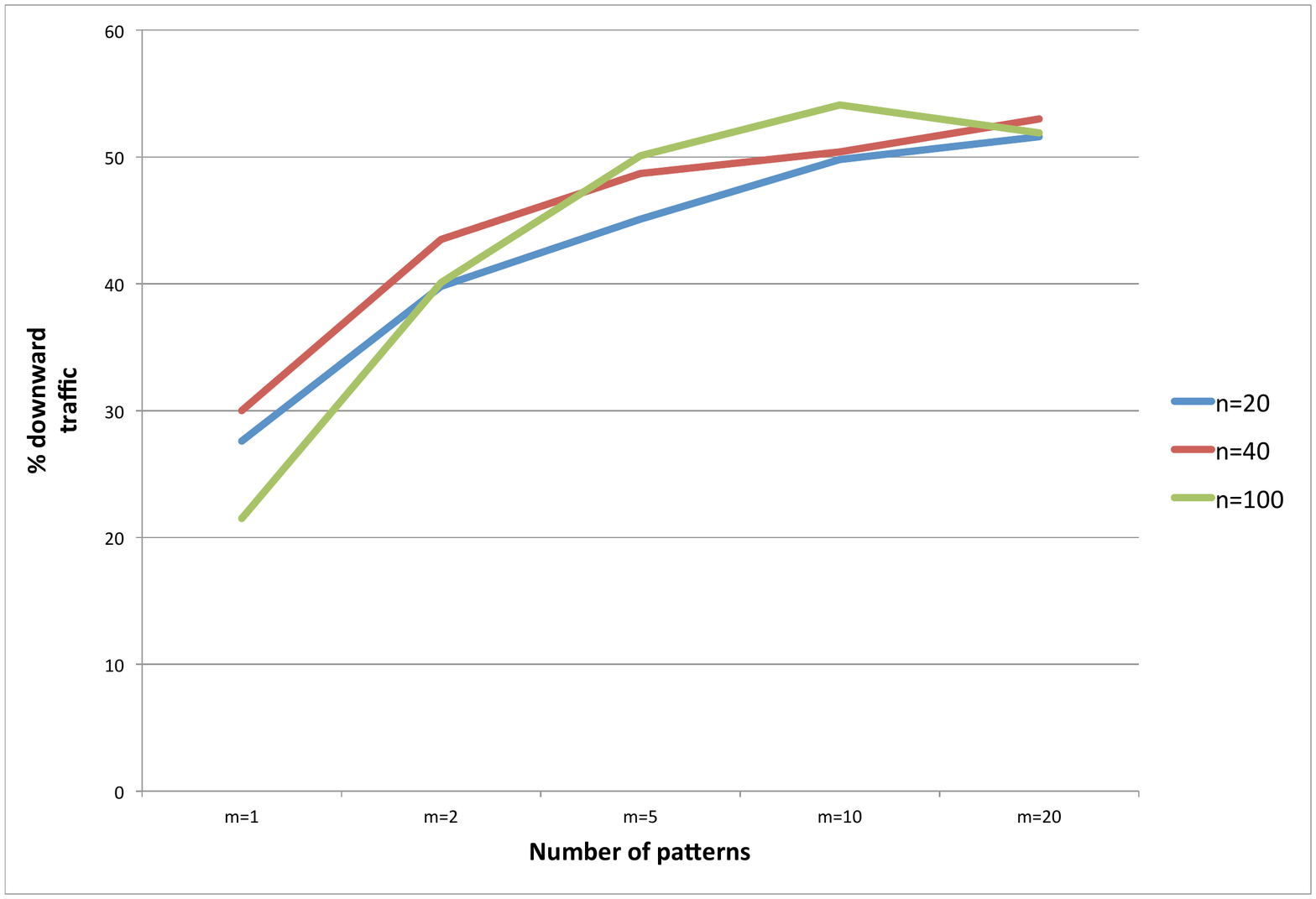}
\includegraphics[width=8.5cm]{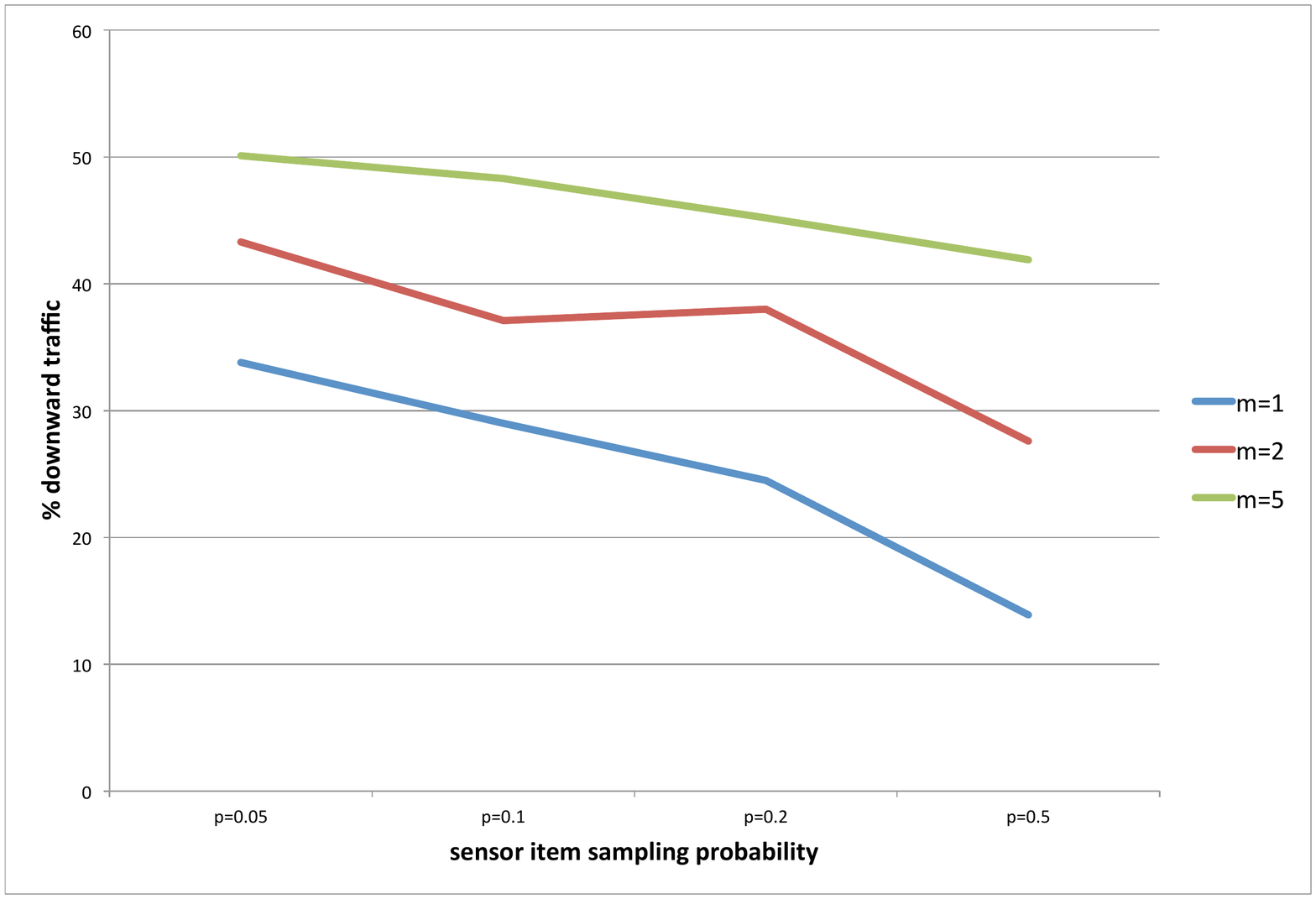}
\caption{\% of downward traffic with (a) more patterns (b) higher sensor sampling probability.}
\label{fig:downward}
\end{figure}

Next, items created in early presentations are relatively stable, with few PJOINS created on later presentations, as predicted by Theorem \ref{thm:learn}. 
Figure \ref{fig:newpjoins} shows the average number of new PJOIN's created in each presentation, as the number of presentations increases, and as the delay for PJOIN eligibility (for items that have fired) is increased. For these experiments we kept $n$ at 40 and 
$p=q=0.1$.

\begin{figure}
\includegraphics[width=8.5cm]{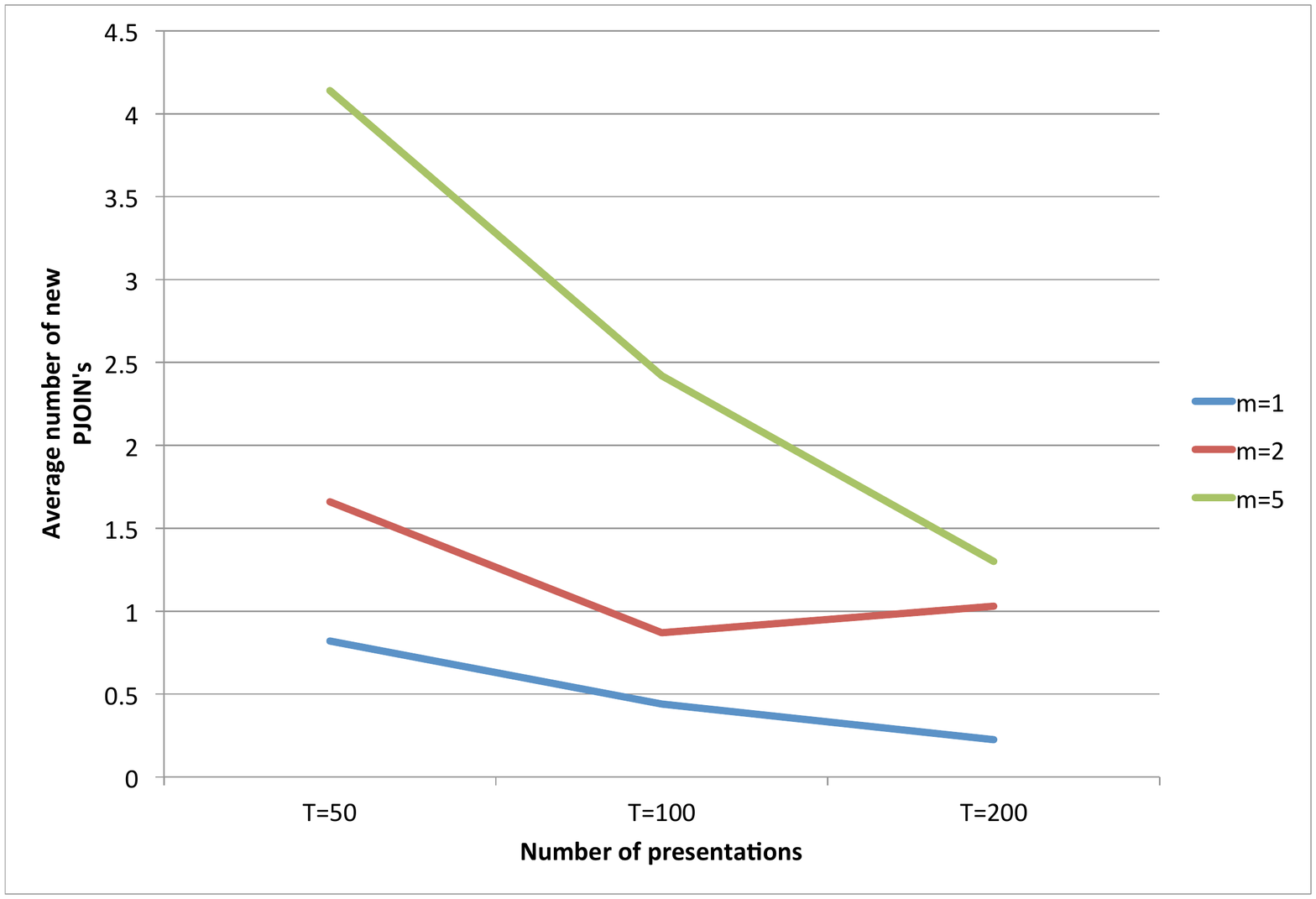}
\includegraphics[width=8.5cm]{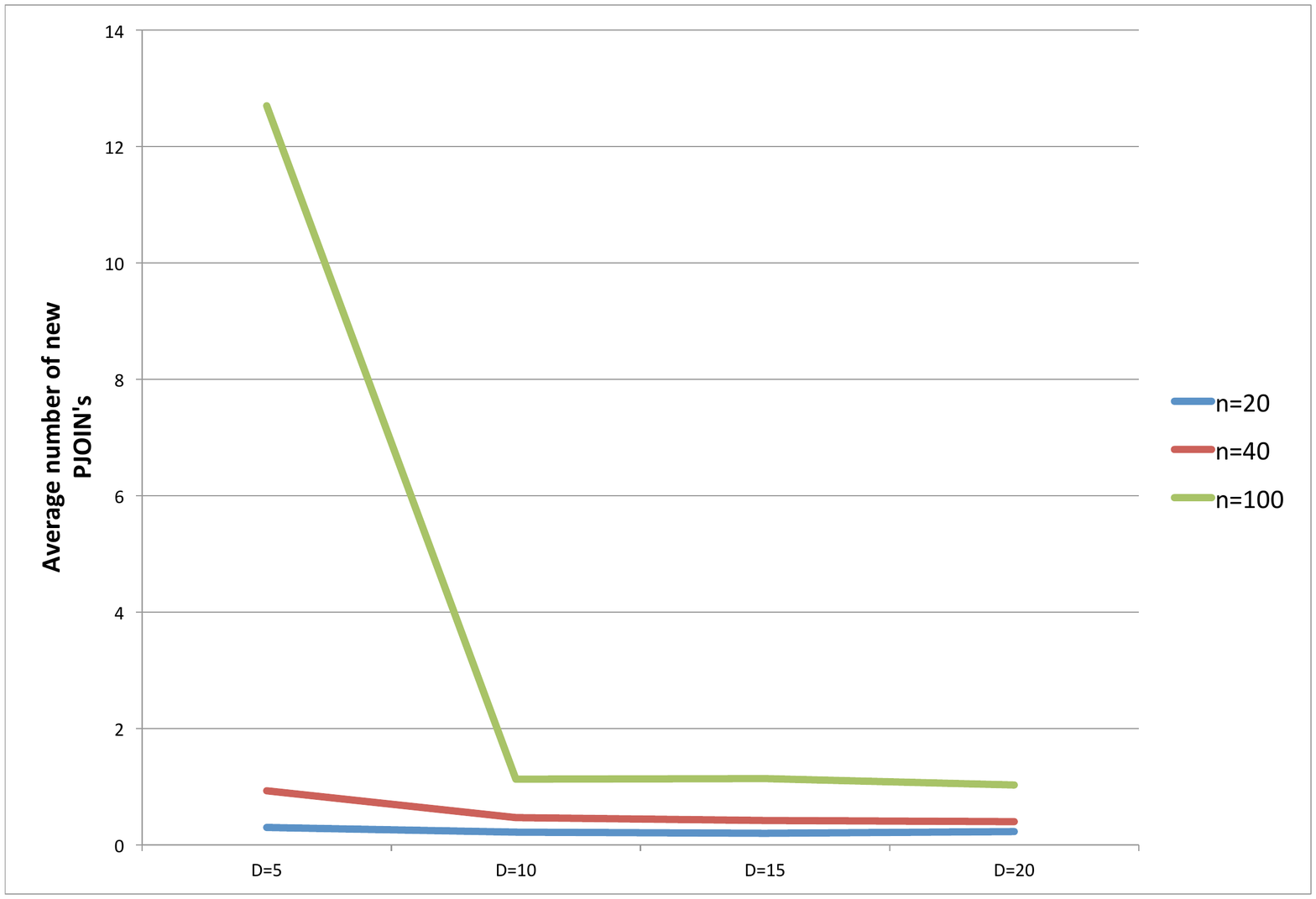}
\caption{Average number of new PJOIN's per presentation, with (a) increasing number of presentations (b) higher delay for PJOIN eligibility.}
\label{fig:newpjoins}
\end{figure}

Another experiment explores the sharing of common substructures by patterns.  We generated the patterns by starting with one random pattern with $n=100$, and obtain $9$ more patterns by perturbing each coordinate of it randomly with probability $p$. In other words, the Hamming distance between consecutive patterns was $p$ in expectation. 
We then used $10$ presentations drawn randomly from this set of patterns.
The perturbation probability $p$ varied from $0.05$ to $0.5$ (the upper bound produces completely random patterns). The results in Figure \ref{fig:perturb}  show that, as perhaps it could have been predicted, when the perturbation probability decreases (that is, when the patters are more similar), the total size of the structure created also decreases, as patterns share more substructures.

\begin{figure}[h]
\begin{center}
\includegraphics[width=9cm]{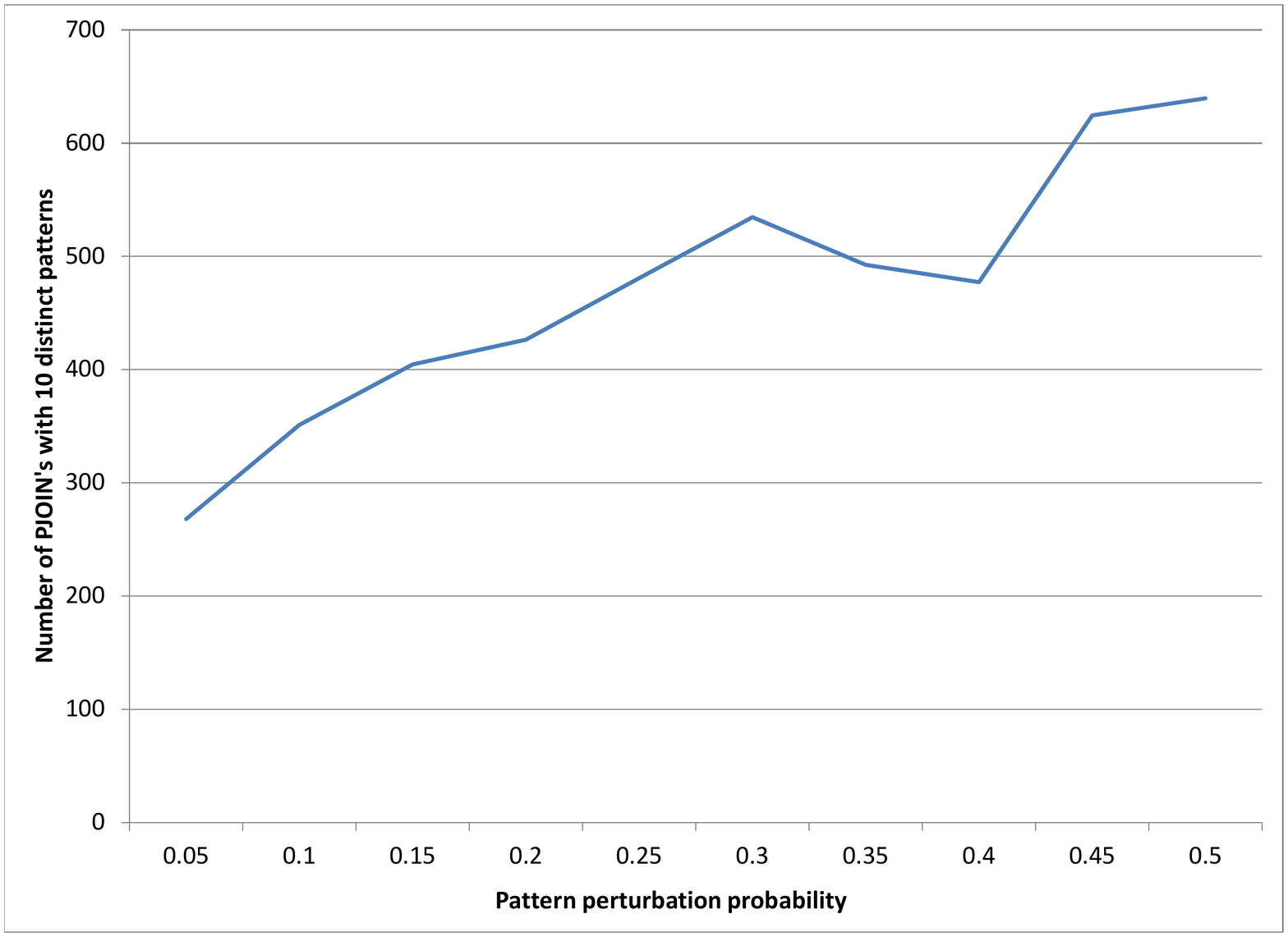}
\caption{Number of PJOINs created, with decreasing overlap between patterns}
\end{center}
\label{fig:perturb}
\end{figure}

Finally, one may wonder whether unsupervised learning could be done more effectively by making prediction {\em random}.  If a parent of an item $C=\PJ(A,B)$ fires downwards, we currently require $C_P$ to fire downwards to both $A$ and $B$.  One interesting variant would only fire downwards to {\em one} of $A$ and $B$, chosen at random.  This new version of $\PJ$ would generate considerably less downward traffic, and detect new patterns at a slower rate than the current one.

\section{Some Remarks on Control}\label{sec:control}
Vicinal algorithms are intended to be an austere, and therefore realistic, model of computation by the neurons in cortex.  And yet certain aspects of the way we are using vicinal algorithms in this work require some further discussion and justification.

Perhaps the most striking aspect of the use of vicinal algorithms on items, as pioneered by Valiant, is the assumption that the $r$ neuroids of an item can fire in synchrony, for example to initiate the creation of a JOIN or a PJOIN.  We refer to this extraneous ability as ``control.''  In the current section we briefly discuss how such control could be plausibly implemented within the vicinal model.  

\subsection*{Reciprocity and a modified random graph model}\label{sec:reciprocity}
Valiant's basic model assumes that neuroids are connected by a random directed graph of synapses in the $G_{np}$ model in which all possible edges have the same probability $p$ of being present.   Measurements of the cortex, however, suggest a more involved model. In particular, Song et a.l \cite{Song05} and subsequent work show that synaptic reciprocity (the chance that there is an edge from $i$ to $j$ {\em and  from $j$ to $i$\/}) is far greater than random.  This is established
for pairs chosen from quadruple whole-cell recordings, where the cells tend to be within close proximity of each other. 
The same is true of synaptic transitivity, i.e., the chance that two nodes connected by a path of length 2 are also connected directly.  

A simple extension of the classical $G_{n,p}$ random graph model can account reasonably well for these departures from randomness. We propose the following:  We start with $n$ vertices.  For every pair of vertices $i,j$ within a distance threshold 
(in the case of the data from \cite{Song05}, this applies to every pair), 
with probability $p$ only the $i \rightarrow j$ edge is in the graph; with probability $p$ only $j \rightarrow i$; and with probability $q$ {\em both} are in the graph; we assume that $2p+q<1$.  With the remaining $1-(2p+q)$ probability, no arc exists  between $i$ and $j$. 

This is already a useful model, incorporating the apparent reciprocity of cortical connections. A more elaborate {\em two-round model} also introduces transitivity.   The first round is as above.  In the second round, for each pair $i,j$ with no arc between $i$ and $j$, for every other proximal vertex $k$, if there is a path of length two between $i$ and $j$ from the first round, ignoring direction, then we add $i \rightarrow j$ with probability $r \cdot p/(2p+q)$, same for $j \rightarrow i$, and we add both arcs with probability $r \cdot q/(2p+q)$. 
This two-round model, with the specific values of  $p=0.05, q =0.047$ and $r=0.15$ gives a reasonable fit with the data in \cite{Song05} (see Figure \ref{fig:Gnpqr}).

\begin{figure}
\begin{center}
\includegraphics[width=10cm]{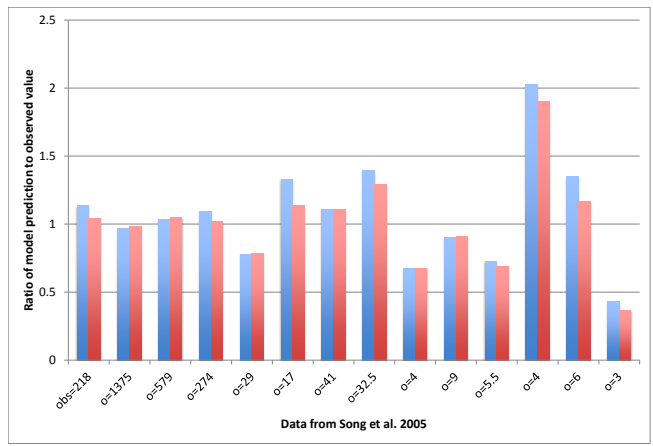}
\caption{Comparison of the random model with reciprocity data. The first bar in each pair is with $p=q=0.05$, the second with$p=0.05, q=0.047$.}
\end{center}
\label{fig:Gnpqr}
\end{figure}

\subsection*{An alternative implementation of $\PJ(A,B)$}
Once we assume that the neuroids have rich reciprocal connections between them, many possibilities arise.  For a first example, the creation of $\PJ(A,B)$ can be simplified a little.  After the neuroids of $C$ have been chosen, $C_P$ could consist of all neuroids in $C$ which happen to have reciprocal connections back to $A_P$ and $B_P$.  Now the $\L$ step is unnecessary, and the creation of a PJOIN can be carried out in three steps, as opposed to four.

\subsection*{Using reciprocity for the control of items}
Consider an item $A$.  Is there a neuroid which is connected to each of the $r$ neuroids  through a path of reciprocal connections of the same small length, call it $\ell$?  For reasonable values of $n$ (total number of available neuroids, say $10^9$), $r$ (number of neuroids in an item, say 50-100), and $q$ (probability that two neuroids are connected in both directions, say $10^{-5}$) it is not hard to see that such a neuroid is very likely to exist, with $\ell$ is bounded by 
\[
\frac{r-1}{r}\cdot\frac{\log n}{\log q n} \le 3.
\] 
Furthermore, we can suppose that such a neuroid $R_A$  (``the root of item $A$'') is discovered by a vicinal process soon after the establishment of item $A$.   One can then imagine that, once the neuroids of $A$ fire, their root can sense this, and make them fire synchronously in a future step, to create a $\PJ(A,B)$, as required by our algorithm.  The delay parameter $D$ of our memorization algorithm may now be seen as the required number of steps for this synchronization to be set up. 

\subsection*{Simultaneous PJOINs}
Another important step of our memorization algorithm, whose implementation in cortex may seem problematic, is the selection of items to be PJOINed.  How are they selected, and how are they matched in pairs?   And, furthermore, won't the {\sc Poised} neuroids (recall the creation of a JOIN described in Section 2) be ``confused'' by all other items firing at once?  Valiant's algorithm for creating JOINs assumes implicitly that nothing else is happening during these two steps; the simultaneous creation of several JOINs requires some thought and explanation.

Suppose that several PJOIN-eligible items fire, presumably in response to control signal from their respective roots, as hypothesized earlier.  It is reasonable to assume that these firings are not simultaneous, but occur at times differing by a fraction of the duration of a step.  One step after the first firing of an item $A$, and in response to that firing, several neuroids are {\sc Poised} to become part of a PJOIN of $A$ with another item.  If we represent the firing of item $A$ just by the symbol $A$, and the time of first availability of the {\sc Poised} items for it as $A'$, a sequence of events during our memorization algorithm can look like this:
$$\ldots A,\ B,\ C,\ A',\  D,\  B',\ C',\ E,\ F,\ D',\ G\   \ldots$$
where $A'$ happens one step after $A$ and similarly for $B'$ after $B$, etc.  Now this sequence of events will result in the following PJOINs being formed: $\PJ(A,D),\ \PJ(B,E),\ \PJ(C,E), \PJ(D,G)$.  Notice that $E$ formed two PJOINs, with both $B$ and $C$, and $F$ formed none, and this is of course fine.   This way, the spontaneous formation of PJOINs envisioned by our algorithm can happen without undesired ``interference'' between different PJOINing pairs, and with only minimal control.  

We believe that this alternative way to select items to be PJOINed in our learning algorithm (by choosing {\em random stars} joining existing items, instead of random edges)  enjoys the same favorable quantitative performance properties established by Theorem \ref{thm:learn}.


\section{Conclusion and Further Work}
We have introduced a new primitive, PJOIN, intended to capture the combining and predicting activity apparently taking place in cortex, and which can be implemented in Valiant's minimalistic model of vicinal algorithms. We showed that, by allowing items to spontaneously form PJOINs, starting from sensory inputs, complex patterns can be memorized and later recognized within very reasonable time bounds.  Much of the activity in this pattern recognition process consists of predicting unseen parts of the pattern, and is directed ``downwards'' in the hierarchy implied by PJOINs, in agreement with current understanding. 

This work touches on the world's most challenging scientific problem, so there is certainly no dearth of momentous questions lying ahead.  Here we only point to a few of them that we see as being close to the nature and spirit of this work.

\begin{itemize}
\item {\bf Invariance.}  We showed how PJOINs can be a part of the solution of a rather low-level problem in cognition, namely the memorization and recognition of patterns.  Naturally, there are much more challenging cognitive tasks one could consider, and we discuss here a next step that is both natural and ambitious.  One of the most impressive things our brain seems to accomplish is the creation of {\em invariants}, higher-level items capturing the fact that all the various views of the same object from different angles and in different contexts (but also, at an even higher level, all sounds associated with that object or person, or even all references to it in language, etc.), all refer to one and the same thing.  It would be remarkable if there is simple and plausible machinery, analogous to PJOIN, that can accomplish this seemingly miraculous clustering operation.  Incidentally, {\em language and grammar} seem to us to be important challenges that are related to the problem of invariance. 

\item {\bf A PJOIN machine?}  This work, besides introducing a useful cortical primitive, can be also seen as a particular stance on cortical computation:  Can it be that the amazing accomplishments of the brain can be traced to some very simple and algorithmically unsophisticated primitives, which however take place at a huge scale?  

\item {\bf Modeling the environment.}  If indeed cortical computation relies not on sophisticated algorithms but on crude  primitives like PJOIN, then the reasons for its success must be sought elsewhere, and most probably in two things:  First, the {\em environment} within which the mammalian cortex has accomplished so much. And secondly, the rich {\em interfaces} with this environment, through sensors and actuators and lower-level neural circuits, resulting, among other things, in complex feedback loops, and eventually modifications of the environment.  What seems to be needed here is a new theory of probabilistic models capable of characterizing the kinds of ``organic'' environments appearing around us, which can result --- through interactions with simple machinery, and presumably through evolution --- to an ascending spiral of complexity in cognitive behavior.
\end{itemize}

\noindent
{\bf Acknowledgements.} We are grateful to Tatiana Emmanouil and Sebastian Seung for useful references, and to Vitaly Feldman and Les Valiant for helpful feedback on an earlier version of this paper.

\bibliographystyle{plain}
\bibliography{neuro}

\end{document}